\documentclass[conference]{IEEEtran}
\usepackage{times}

\usepackage[numbers]{natbib}
\usepackage{multicol}
\usepackage{graphicx}
\usepackage{amsfonts}
\usepackage{amsmath}
\usepackage{algorithm}
\usepackage{algorithmic}
\usepackage{amsthm}
\usepackage{url}
\usepackage{float}
\usepackage{subfig}

\newtheorem{theorem}{Theorem}[section]

\begin{document}

\title{Exploration in Action Space}


\author{\authorblockN{Anirudh Vemula}
\authorblockA{Robotics Institute\\
Carnegie Mellon University\\
Pittsburgh, Pennsylvania 15213\\
Email: vemula@cmu.edu}
\and
\authorblockN{Wen Sun}
\authorblockA{Robotics Institute\\
Carnegie Mellon University\\
Pittsburgh, Pennsylvania 15213\\
Email: wensun@cs.cmu.edu}
\and
\authorblockN{J. Andrew Bagnell}
\authorblockA{Robotics Institute\\
Carnegie Mellon University\\
Pittsburgh, Pennsylvania 15213\\
Email: dbagnell@ri.cmu.edu}}


%

\maketitle

\begin{abstract}
Parameter space exploration methods with black-box optimization have
recently been shown to outperform state-of-the-art approaches in
continuous control reinforcement learning domains. In this paper, we examine reasons why
these methods work better and the situations
in which they are worse than traditional action space exploration
methods. Through a simple theoretical analysis, we show that when the parametric complexity required to solve
the reinforcement learning problem is greater than the
product of 
action space dimensionality and horizon length, exploration
in action space is preferred. 
This is also shown empirically by comparing simple exploration methods
on several toy problems.

\end{abstract}

\IEEEpeerreviewmaketitle

\section{Introduction}

Recently, in a series of blog posts\footnote{\url{http://www.argmin.net/2018/03/20/mujocoloco/}} and in \cite{mania2018simple}, Ben Recht and colleagues reached the following conclusion:
``Our findings contradict the common belief that policy gradient techniques, which rely on exploration in the action space, are more sample efficient than methods based on
finite-differences.''

That's a conclusion that we have often felt has much merit. In a survey (with Jens Kober and Jan Peters) \cite{kober2013reinforcement}, we wrote: 

\emph{
``Black box methods are general stochastic optimization algorithms (Spall, 2003) using only the expected return of policies, estimated by sampling, and do not leverage any of the internal structure of the RL problem. These may be very sophisticated techniques (Tesch et al., 2011) that use response surface estimates and bandit-like strategies to achieve good performance. White box methods take advantage of some of additional structure within the reinforcement learning domain, including, for instance, the (approximate) Markov structure of problems, developing approximate models, value-function estimates when available (Peters and Schaal, 2008c), or even simply the causal ordering of actions and rewards. A major open issue within the field is the relative merits of the these two approaches: in principle, white box methods leverage more information, but with the exception of models (which have been demonstrated repeatedly to often make tremendous performance improvements, see Section 6), the performance gains are traded-off with additional assumptions that may be violated and less mature optimization algorithms. Some recent work including (Stulp and Sigaud, 2012; Tesch et al., 2011) suggest that much of the benefit of policy search is achieved by black-box methods.''
}

Many empirical examples--the classic Tetris \cite{Thiery2009BuildingCF} for instance--demonstrate that Cross-Entropy or other (fast) black-box heuristic methods are generally far superior to any policy gradient method and that policy gradient methods often achieve orders of magnitude better performance than methods demanding still more structure like, \textit{e.g.} temporal difference learning \cite{Sutton1998}.  
In this paper, we set out to study why black-box parameter space exploration methods work better and in what situations can we expect them to perform worse than traditional action space exploration methods.

\section{The Structure of Policies}

Action space exploration methods, like REINFORCE \cite{williams1992simple}, SEARN \cite{daume2009search}, PSDP \cite{bagnell2004policy}, AGGREVATE \cite{ross2014reinforcement,sun2017deeply}, LOLS \cite{chang2015learning}, leverage more structure than parameter space exploration methods. More specifically, 
they understand the relationship (\textit{e.g.},  the Jacobian) between a policy's parameters and its outputs.  We could ask: \textit{Does this matter?} In the regime of large parameter space and small output space problems as explored often by our colleagues, like atari games \cite{mnih2015human}, it might.
(Typical implementations also leverage causality of reward structure as well, although one might expect that is relatively minor.)

In particular, the intuition behind the use of action space exploration techniques is that they should perform well when the action space is quite small compared to the parametric complexity
required to solve the Reinforcement Learning problem.

\section{Experiments}
We test this intuition across three experiments: MNIST, Linear Regression and LQR. The code for all these experiments is published at \url{https://github.com/LAIRLAB/ARS-experiments}.

\subsection{MNIST}

To investigate this we, like \cite{argmin}, 
consider some toy RL problems, beginning with a single time step MDP.
In particular, we start with a classic problem: MNIST digit recognition. To put in a bandit/RL framework, we consider a $+1$ reward for getting the digit correct and
a $-1$ reward for getting it wrong. We use a LeNet-style architecture\footnote{Two convolution layers each with $5\times 5$ kernels with $10$ and $20$ output channels, followed by two fully connected layers of $320$ and $50$ units and a output softmax layer resulting in 1-hot encoding of dimensionality $10$}, \cite{lecun1998gradient}. The total number of trainable parameters in this architecture is $d = 21840$. The experimental setup is described in greater detail in Appendix \ref{sec:mnist-details}.

We then compare the learning curves for vanilla REINFORCE 
with supervised learning and with ARS \textbf{V2-t}, an augmented random search in parameter space procedure introduced in \citet{mania2018simple}. Figure \ref{fig:mnist} demonstrates the results where solid lines represent mean test accuracy over $10$ random seeds and the shaded region corresponds to $\pm 1$ standard deviation. Clearly, in this setting where the parameter space dimensionality significantly exceeds the action space dimensionality, we can see that action space exploration methods such as REINFORCE outperform parameter space exploration methods like ARS.

\begin{figure}[t]
    \centering
    \includegraphics[width=0.9\linewidth]{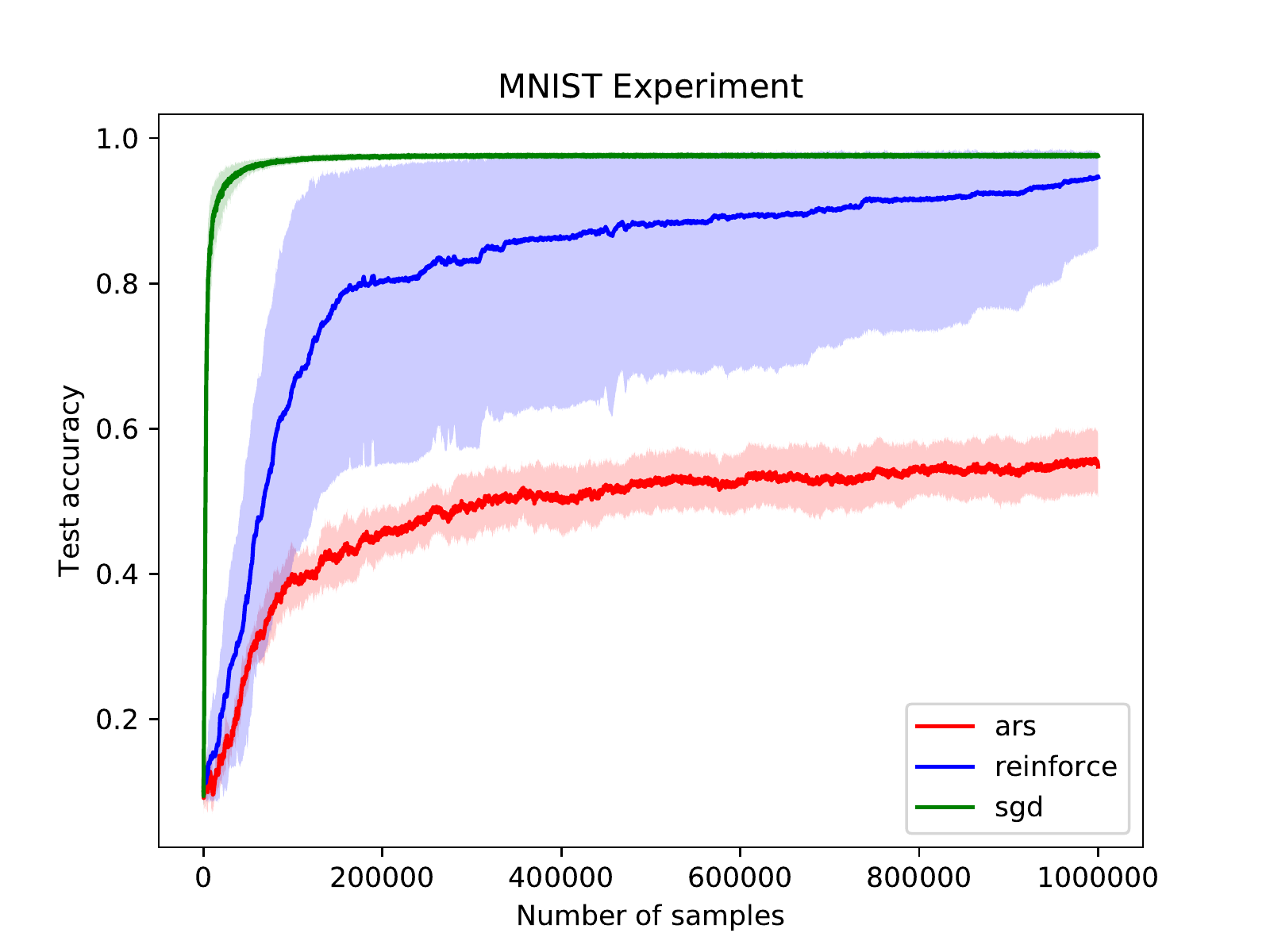}
    \caption{Test accuracy of different approaches against number of samples}
    \label{fig:mnist}
\end{figure}

\subsection{Linear Regression}

Following the heuristic of the \textit{Linearization Principle} introduced in \cite{argmin},
we attempt to get a handle on the trade-off between sample complexity in parameter space and complexity in action space by considering another simple, one step problem: linear regression
with a \emph{single} output variable and $d$ input dimensionality, and thus $d$ parameters. We consider learning a random linear function. 

Before empirically studying REINFORCE and ARS, we first perform a regret analysis (Appendix \ref{sec:linreg-analysis}) on online linear regression for three learning strategies: (1) online gradient descent, (2) exploration in action space, and (3) exploration in parameter space, which correspond to full information setting, linear contextual bandit setting, and bandit setting respectively. The online gradient descent approach is simply OGD from \citet{Zinkevich2003_ICML} applied to full information online linear regression setting, and for exploration in parameter space, we simply used the BGD algorithm from \citet{flaxman2005online}, which completely ignores the properties of linear regression setting and works in a bandit setting. The algorithm for random exploration in action space---possibly the simplest linear contextual bandit algorithm, shown in Alg.~\ref{alg:random_search_action}, operates in the middle: it has access to  feature vectors and performs random search in prediction space to form estimations of gradients.  The analysis of all three algorithms is performed in the online setting: no statistical assumptions on the sequence of linear loss functions, and multi-point query per loss function \cite{agarwal2010optimal} is not allowed.\footnote{Note that the ARS algorithms presented in \citet{mania2018simple} actually take advantage of the reset property of episodic RL setting and perform two-point feedback query to reduce the variance of gradient estimations}

The detailed algorithms and analysis are provided in Appendix~\ref{sec:linreg-details}. The main difference between exploration in action space and exploration in parameter space is that exploration in action space can take advantage of the fact that the predictor it is learning is \emph{linear} and it has access to the linear feature vector (i.e., the Jacobian of the predictor). The key advantage of exploration in action space over exploration in parameter space is that \emph{exploration in action space is input-dimension free}.\footnote{It will dependent on the output dimension, if one considers multivariate linear regression.} More specifically, one can show that in order to achieve $\epsilon$ average regret, the algorithm (Alg.~\ref{alg:random_search}) performing exploration in parameter space requires $O(\frac{d^2}{\epsilon^4})$ samples (we ignore problem specific parameters such as the maximum norm of feature vector, the maximum norm of the linear predictor, and the maximum value of prediction, which we assume are constants), while the algorithm (Alg.~\ref{alg:random_search_action}) performs exploration in action space requires $O(\frac{1}{\epsilon^4})$, which is not explicitly dependent on $d$. 

We empirically compare the test squared loss of REINFORCE, natural REINFORCE (which simply amounts to whitening of input features) \cite{kakade2002natural} and the ARS \textbf{V2-t} method discussed in \citet{mania2018simple} with
classic follow-the-regularized-leader (Supervised Learning). The results are shown in Figures \ref{fig:lin10}, \ref{fig:lin100} and \ref{fig:lin1000}, where solid lines represent mean test squared loss over 10 random seeds and the shaded region corresponds to $\pm 1$ standard deviation. The learning curves match our expectations, and moreover show that this bandit style REINFORCE lies between the curves of supervised learning and parameter space exploration: that is action space exploration takes advantage of the Jacobian of the policy itself and can learn much more quickly.

\begin{figure*}[t]
  \centering
  \subfloat[$d=10$]{\includegraphics[width=0.3\linewidth]{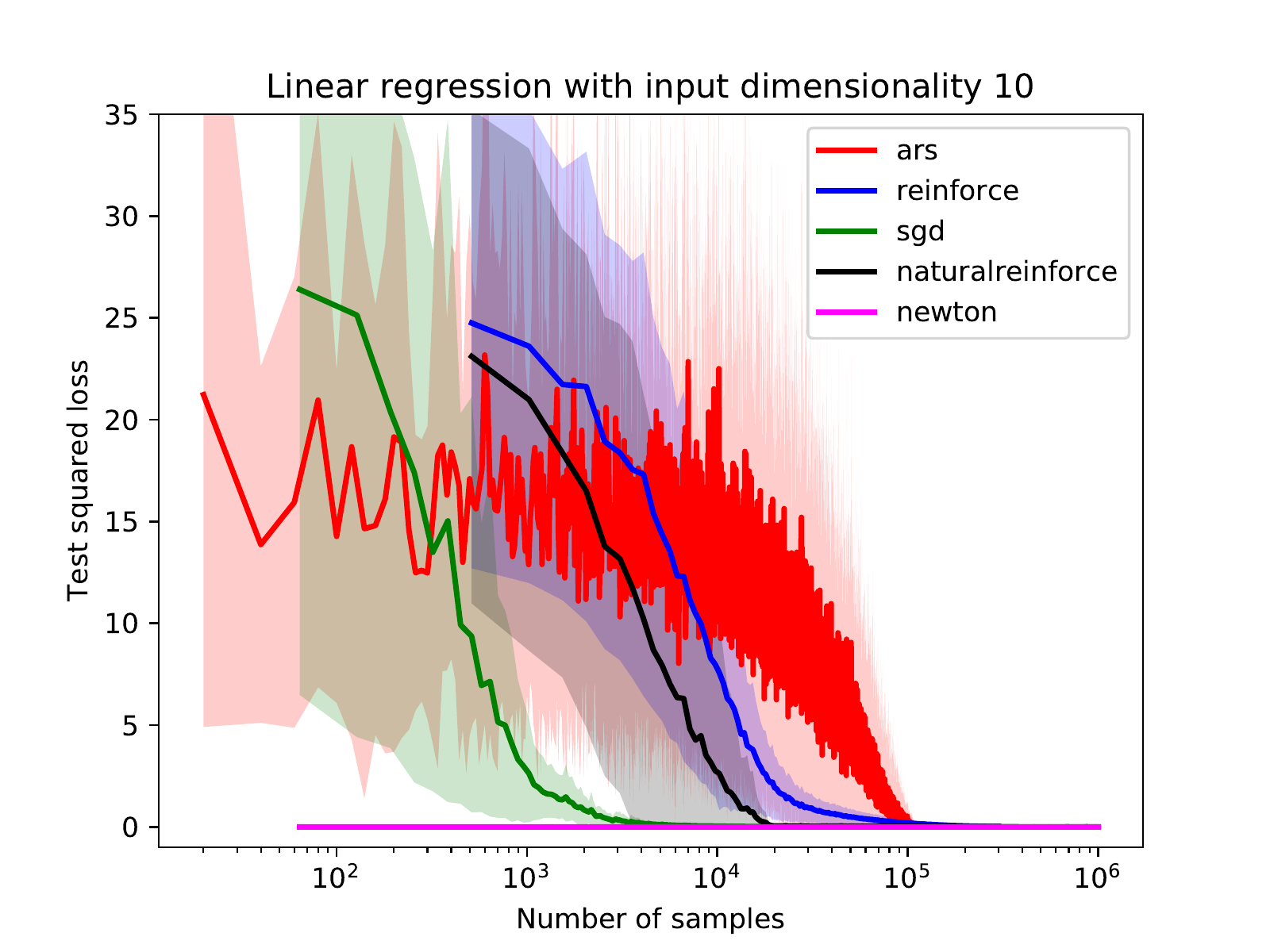}\label{fig:lin10}}
  \subfloat[$d=100$]{\includegraphics[width=0.3\linewidth]{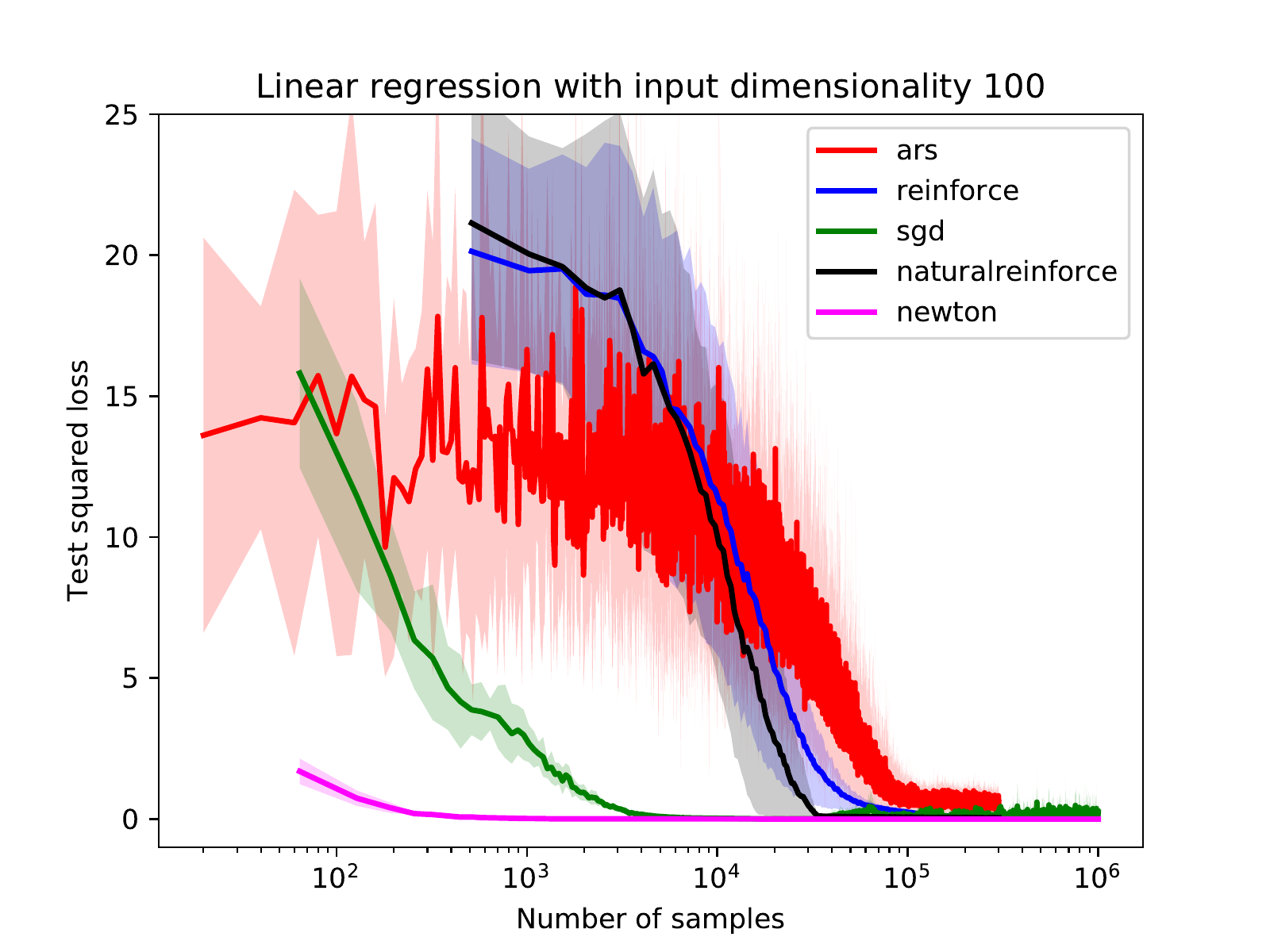}\label{fig:lin100}}
  \subfloat[$d=1000$]{\includegraphics[width=0.3\linewidth]{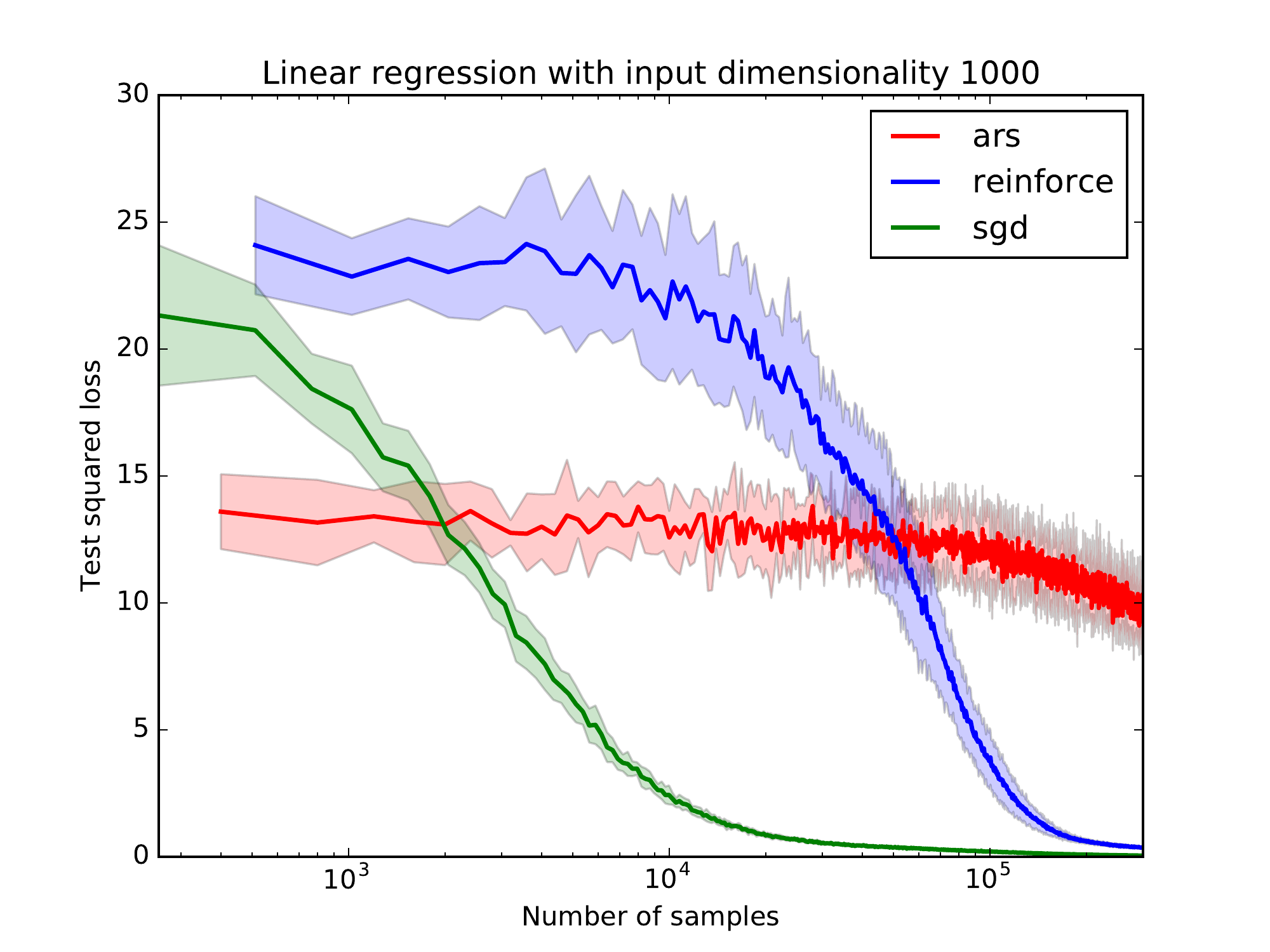}\label{fig:lin1000}}
  \caption{Linear Regression Experiments with varying input dimensionality}
\end{figure*}

\subsection{LQR}

Finally, we consider what happens as we extend the time horizon. In Appendix~\ref{sec:analysis-rl}, we consider finite horizon ($H>1$) optimal control task with deterministic dynamics, fixed initial state and a linear stationary policy. We show that we can estimate the policy gradient via a random search in parameter space as ARS did (Eq.~\ref{eq:parameter_space_rl} in Appendix~\ref{sec:analysis-rl}), or we can do a random search in action space across all time steps independently (Eq.~\ref{eq:action_space_rl} in Appendix~\ref{sec:analysis-rl}). Comparing the norm of both gradient estimates, we can see that the major difference is that the norm of the gradient estimate from random exploration in parameter space (Eq.~\ref{eq:parameter_space_rl}) linearly scales with the dimensionality of state space (i.e., dimensionality of parameter space as we assume linear policy), while the norm of the gradient estimate from random search in action space (Eq.~\ref{eq:action_space_rl}) linearly scales with the product of horizon and action space dimensionality. Hence, when the dimensionality of the state space is smaller than the product of horizon and action space dimensionality, one may prefer random search in parameter space, otherwise random search in action space is preferable. Note that for most of the continuous control tasks in OpenAI gym \cite{openaigym}, the horizon is significantly larger than the state space dimensionality.\footnote{Take Walker2d-v2 as an example, $H$ is usually equal to 1000. The action space dimension is $6$, and the dimension of the state space is $17$. Hence random exploration in action space is actually randomly searching in 6000 dimension space, while random search in parameter space is searching in 17 dimension space.} This explains why ARS \cite{mania2018simple} outperforms most of the action space exploration methods in these tasks.


The simplest setting to empirically evaluate this is where we'd all likely agree that a model based method would be the preferred approach: a finite-horizon Linear Quadratic Regulator problem with 1-d control space and a $100$-d state space.  We then compare random search (ARS \textbf{V1-t} from \citet{mania2018simple}) versus REINFORCE (with ADAM \cite{kingma2014adam} as the underlying optimizer), in terms of the number of samples they need to train a stationary policy that reaches within 5\% error of the non-stationary optimal policy's performance with respect to the horizon $H$ ranging from $10$ to $160$. Fig.~\ref{fig:LQG} shows the comparison where the statistics are averaged over 10 random seeds (mean $\pm$ standard error). 

\begin{figure}[t]
    \centering
    \includegraphics[width=\linewidth]{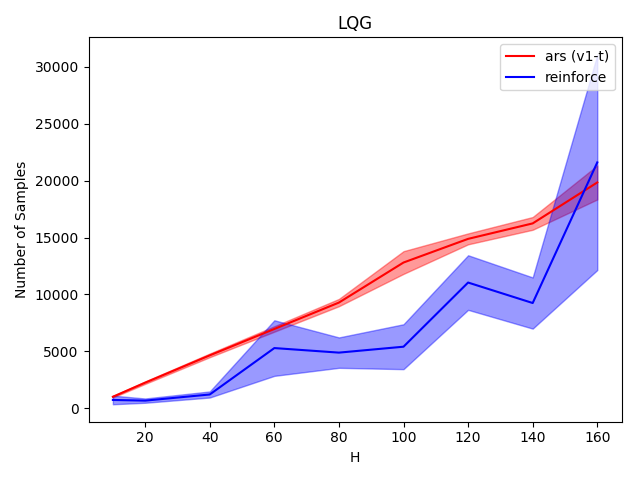}
    \caption{LQG with state dimensionality $d = 100$ with varying horizon $H$}
    \label{fig:LQG}
\end{figure}  

From Fig.~\ref{fig:LQG} we see that as $H$ increases, both algorithms require larger number of samples.\footnote{Our simple analysis on finite horizon optimal control with deterministic dynamic shows that both gradient estimators' norm linearly depends on the objective value measured at the current policy. As $H$ increases, the total cost increases.} While ARS is stable across different random seeds, we found reinforce becomes more and more sensitive to random seeds when $H$ is larger, and performance becomes less stable. However notice that when $H$ is small (e.g., $H\leq40$), we see that REINFORCE has lower variance as well and can consistently outperform ARS. Though we would expect that REINFORCE would require more samples than ARS when $H\geq 100$, which is the dimension of state space, in this experiment we only notice this phenomenon at $H=160$, though the variance of REINFORCE is already terrible at that point.


\section{Conclusion}
\label{sec:conclusion}
\citet{mania2018simple} have shown that simple random search in parameter space is a competitive alternative to traditional action space exploration methods. In our work, we show that this is only true for domains where the parameter space dimensionality is significantly smaller than the product of action space dimensionality and horizon length. For domains where this does not hold true, action space exploration methods such as REINFORCE \cite{williams1992simple} are more sample efficient as they do not have an explicit dependence on parameter space dimensionality.

\section{Acknowledgements}

We thank the entire LairLab for stimulating discussions, and Ben Recht for his interesting blog posts.



\bibliographystyle{plainnat}
\bibliography{reference}

\appendix

\subsection{Detailed Algorithms and Analysis for Linear Regression}
\label{sec:linreg-analysis}

We present our algorithms in the online learning setting with square loss function at time step $t$ as $\ell_t(w) = (w^T x_t - y_t)^2$, where $x_t\in\mathbb{R}^d, y_t\in\mathbb{R}$. We design three different levels of feedback: (1) For classic supervised learning, at each round the learner receives feature $x_t$ and then the learner makes a prediction $w_t^T x_t$; it will receives the loss $\ell_t(w_t) = (\hat{y}_t - y_t)^2$ together with $y_t$; (2) For random search in parameter space (e.g., ARS), we assume we operate in a full bandit setting: once the learner proposes some predictor $w_t$, it will only receives the loss measured at $w_t$: $\ell_t(w_t) = (w_t^T x_t- y_t)^2$ without explicit information on $x_t$ or $y_t$; (3) For random search in action space (e.g., REINFORCE), we operate in a linear contextual bandit setting: the learner receives feature $x_t$, and then the learner makes prediction $\hat{y}_t = w_t^T x_t$, it receives the loss measured at $\hat{y}_t$, i.e., $(y_t-\hat{y}_t)^2$ but no $y_t$. For simplicity in analysis, we assume the learner always chooses prediction $w$ from a pre-defined convex set $\mathbb{W}$: $w\in \mathbb{W}\subseteq \mathbb{R}^d$. 

Below we present three algorithms. The first algorithm, OGD from \cite{Zinkevich2003_ICML}, shown in Alg.~\ref{alg:Linear_ogd}, simply operates in the full information setting; the second algorithm, Random Search in Parameter Space, adopted from \cite{flaxman2005online},\footnote{The algorithm presented in \cite{flaxman2005online} deals with projection more rigorously: to make sure that the randomly perturbed predictor $w_t+\delta u$ also belongs to $\mathbb{W}$, we have to project $w_{t}'$ to a slightly shrunk convex set $(1-\alpha) \mathbb{W}$ with $\alpha\in (0,1)$. Here we simply assume it is acceptable to step outside $\mathbb{W}$ a little bit since $\delta$ will be extremely small when $T$ is large.} presented in Alg.~\ref{alg:random_search}, operates in a full bandit setting; the third algorithm, Random Search in Action Space, presented in Alg.~\ref{alg:random_search_action}, operates in a classic linear contextual bandit setting. Below we present the regret analysis for these three algorithms. 

\begin{algorithm}[h]
\caption{Online Gradient Descent \cite{Zinkevich2003_ICML}}
 \label{alg:Linear_ogd}
\begin{algorithmic}[1]
  \STATE {\bfseries Input:} Learning rate $\mu_t \in \mathbb{R}^+$
  \STATE Learner initializes $w_0\in\mathbb{W}$.
  \FOR {$t = 0$ to $T$}
    \STATE Learner receives $x_t$
    \STATE Learner makes prediction $\hat{y}_t = w_t^T x_t$. 
    \STATE Learner receives loss  $(\hat{y}_t - y_t)^2$ and $y_t$.
    \STATE Learner update: $w_{t+1}' = w_t - \mu_t (\hat{y}_t - y_t)x_t$.
    \STATE Projection $w_{t+1} = \arg\min_{w\in\mathbb{W}}\|w_{t+1}' - w\|_2^2$
  \ENDFOR
\end{algorithmic}
\end{algorithm}

\begin{algorithm}[h]
\caption{Random Search in Parameter Space (BGD \cite{flaxman2005online})}
 \label{alg:random_search}
\begin{algorithmic}[1]
  \STATE {\bfseries Input:} Learning rate $\mu_t\in\mathbb{R}^+$, finite difference parameter $\delta\in\mathbb{R}^+$.
  \STATE Learner initializes $w_0\in\mathbb{W}$.
  \FOR {$t = 0$ to $T$}
    \STATE Learner samples $u$ uniformly from the surface of $d$-dim sphere.
    \STATE Learner chooses predictor $w_t + \delta u$. 
    \STATE Learner only receives loss signal $\ell_t$, which is $((w_t+\delta u)^T x_t - y_t)^2$.
    \STATE Learner update: $w_{t+1}' = w_t - \mu_t \frac{\ell_t d}{\delta}u$.
    \STATE Projection $w_{t+1} = \arg\min_{w\in\mathbb{W}}\|w_{t+1}'-w\|_2^2$.
  \ENDFOR
\end{algorithmic}
\end{algorithm}

\begin{algorithm}[h]
\caption{Random Search in Action Space}
 \label{alg:random_search_action}
\begin{algorithmic}[1]
  \STATE {\bfseries Input:} Learning rate $\mu_t\in\mathbb{R}^+$, finite difference parameter $\delta\in\mathbb{R}^+$.
  \STATE Learner initializes $w_0\in\mathbb{W}$.
  \FOR {$t = 0$ to $T$}
    \STATE Learner receives feature $x_t$
    \STATE Learner samples $e$ uniformly from $\{-1,1\}$
    \STATE Learner makes a prediction $\hat{y}_t = w_t^T x_t + \delta e$
    \STATE Learner only receives loss signal $\ell_t$, which is $(\hat{y}_t - y_t)^2$
    \STATE Learner update: $w_{t+1}' = w_t - \mu_t \frac{\ell_te}{\delta}x_t$.
    \STATE Projection $w_{t+1} = \arg\min_{w\in\mathbb{W}}\|w_{t+1}'-w\|_2^2$.
  \ENDFOR
\end{algorithmic}
\end{algorithm}

We assume $w$ is bounded as $\|w\|_2\leq \mathcal{W}\in\mathbb{R}^+$, $x$ is bounded $\|x\|_2\leq \mathcal{X}\in\mathbb{R}^+$, regression target $y$ is bounded: $\|y\|\leq \mathcal{Y}\in\mathbb{R}^+$, to make sure the loss is bounded $|w^Tx - y| \leq C\in\mathbb{R}^+$, and the gradient is bounded $\|\nabla_{w}\ell_t(w)\|_2\leq C\mathcal{X}$. Note that with these assumptions, the loss function $\ell_t(w)$ is Lipschitz continuous with Lipschitz constant $L\leq (\mathcal{W}\mathcal{X}+\mathcal{Y})\mathcal{X}$. We assume that these constants, $\mathcal{W}$, $\mathcal{X}$, $\mathcal{Y}$, are independent of the feature dimension $d$. For the following analysis, we may omit constants that are independent of $d$ and $T$, but will keep the dependency on $d$ or $T$ explicitly.

\begin{theorem}
After $T$ rounds, with $\mu_t = \frac{\mathcal{W}}{C\mathcal{X}\sqrt{t}}$ in Alg.~\ref{alg:Linear_ogd}, Alg.~\ref{alg:Linear_ogd} has regret:
\begin{align}
\label{eq:ogd}
    \sum_{t=0}^T \ell_t(w_t) - \min_{x^*\in\mathbb{W}}\ell_t(w^*) \leq \mathcal{W}C\mathcal{X}\sqrt{T}.
\end{align}
With $\mu_t = \frac{\mathcal{W}\delta}{d(C^2+\mathcal{X}^2)\sqrt{T}}$ and $\delta = T^{-0.25}\sqrt{\frac{\mathcal{W}d(C^2+\mathcal{X}^2)}{2L}}$ in Alg.~\ref{alg:random_search}, Alg.~\ref{alg:random_search} incurs regret:
\begin{align}
\label{eq:random_para}
    \mathbb{E}\left[\sum_{t=0}^T \ell_t(w_t)\right] - \min_{w^*\in\mathbb{W}}\sum_{t=0}^T \ell_t(w^*) \leq \sqrt{\mathcal{W}d(C^2+\mathcal{X}^2)L} T^{3/4}
\end{align}
With $\mu_t = \frac{\mathcal{W}\delta}{(C^2+1)\mathcal{X}\sqrt{T}}$ and $\delta = T^{-0.25}\sqrt{\frac{\mathcal{W}(C^2+1)\mathcal{X}}{2C}}$, Alg.~\ref{alg:random_search_action}, Alg.~\ref{alg:random_search_action} incurs regret:
\begin{align}
\label{eq:random_action}
    \mathbb{E}\left[\sum_{t=0}^T \ell_t(w_t)\right] - \min_{w^*\in\mathbb{W}}\sum_{t=0}^T \ell_t(w^*) \leq \sqrt{\mathcal{W}(C^2+1)\mathcal{X}C}T^{3/4}
\end{align}
\end{theorem}
\begin{proof}
Result in Eq~\ref{eq:ogd} is directly from \cite{Zinkevich2003_ICML} with the fact that $\|w\|_2\leq\mathcal{W}$ and $\|\nabla_{w}\ell_t(w)\|_2\leq C\mathcal{X}$ to any $w$ and $t$. 

To prove Eq.~\ref{eq:random_para} for Alg.~\ref{alg:random_search}, we use the proof techniques from \cite{flaxman2005online}. The proof is more simpler than the one in \cite{flaxman2005online} as we do not have to deal with shrinking and reshaping the predictor set $\mathbb{W}$.

Denote $u\sim \mathbb{B}$ as uniformly sampling $u$ from a $d$-dim unit ball, $u\sim\mathbb{S}$ as uniformly sampling $u$ from the $d$-dim unit sphere, and $\delta \in (0,1)$. Consider the loss function $\hat{\ell}_t(w_t) = \mathbb{E}_{v\sim \mathbb{B}}[\ell_t(w_t + \delta v)]$, which is a smoothed version of $\ell_t(w_t)$. It is shown in \cite{flaxman2005online} that the gradient of $\hat{\ell}_t$ with respect to $w$ is:
\begin{align}
   \nabla_{w}\hat{\ell}_t(w)|_{w = w_t} &= \frac{d}{\delta}
                                          \mathbb{E}_{u\sim\mathbb{S}}[\ell_t(w_t
                                          +\delta u)u] \\
                                        &= \frac{d}{\delta}\mathbb{E}_{u\sim \mathbb{S}}[((w_t +\delta u)^T x_t - y_t)^2 u]
\end{align} Hence, the descent direction we take in Alg.~\ref{alg:random_search} is actually an unbiased estimate of $\nabla_{w}\hat{\ell}_t(w)|_{w=w_t}$. So Alg.~\ref{alg:random_search} can be considered as running OGD with an unbiased estimate of gradient on the sequence of loss $\hat{\ell}_t(w_t)$. It is not hard to show that for an unbiased estimate of $\nabla_{w}\hat{\ell}_t(w)|_{w=w_t}$ = $\frac{d}{\delta} (w_t + \delta u)^T x_t - y_t)^2 u$, the norm is bounded as $d(C^2 + \mathcal{X}^2)/\delta$. Now we can directly applying Lemma 3.1 from \cite{flaxman2005online}, to get:
\begin{align}
\label{eq:regret_on_surrogate}
   \mathbb{E}\left[\sum_{t=0}^T \hat{\ell}_t(w_t)\right] -
  \min_{w^*\in\mathbb{W}}\sum_{t=0}^T \hat{\ell}_t(w^*)
  \leq \frac{\mathcal{W}d(C^2+\mathcal{X}^2)}{\delta}\sqrt{T}
\end{align} We can bound the difference between $\hat{\ell}_t(w)$ and ${\ell}_t(w)$ using the Lipschitiz continuous property of $\ell_t$:
\begin{align}
|\hat{\ell}_t(w) - \ell_t(w) | & = |\mathbb{E}_{v\sim \mathbb{B}}[\ell_t(w+\delta v) - \ell_t(w)]| \nonumber\\
&\leq \mathbb{E}_{v\sim \mathbb{B}}[|\ell_t(w+\delta v) - \ell_t(w)|] \leq L\delta
\end{align} Substitute the above inequality back to Eq.~\ref{eq:regret_on_surrogate}, rearrange terms, we get:
\begin{align}
\mathbb{E}\left[ \sum_{t=0}^T \ell_t(w_t)  \right]  -
  \min_{w^*\in\mathbb{W}} \sum_{t=0}^T \ell_t(w^*) \leq
  &\frac{\mathcal{W}d(C^2+\mathcal{X}^2)}{\delta}\sqrt{T} \nonumber\\
  &+ 2LT\delta.
\end{align} By setting $\delta = T^{-0.25}\sqrt{\frac{\mathcal{W}d(C^2+\mathcal{X}^2)}{2L}}$, we get:
\begin{align}
   \mathbb{E}\left[ \sum_{t=0}^T \ell_t(w_t)  \right]  - \min_{w^*\in\mathbb{W}} \sum_{t=0}^T \ell_t(w^*) \leq \sqrt{\mathcal{W}d(C^2+\mathcal{X}^2)L} T^{3/4}
\end{align}

To prove Eq.~\ref{eq:random_action} for Alg.~\ref{alg:random_search_action}, we follow the similar strategy in the proof of Alg.~\ref{alg:random_search}.

Denote $\epsilon \sim [-1,1]$ as uniformly sampling $\epsilon$ from the interval $[-1,1]$, $e\sim \{-1,1\}$ as uniformly sampling $e$ from the set containing $-1$ and $1$. Consider the loss function $\tilde{\ell}_t(w) = \mathbb{E}_{\epsilon\sim [-1,1]}[(w^T x_t + \delta \epsilon - y_t)^2]$. One can show that the gradient of $\tilde{\ell}_t(w)$ with respect to $w$ is:
\begin{align}
    \nabla_{w}\tilde{\ell}_t(w) = \frac{1}{\delta}\mathbb{E}_{e\sim \{-1,1\}}[e(w^Tx_t + \delta e - y_t)^2 x_t]
\end{align} As we can see that the descent direction we take in Alg.~\ref{alg:random_search_action} is actually an unbiased estimate of $\nabla_{w}\tilde{\ell}_t(w)|_{w=w_t}$. Hence Alg.~\ref{alg:random_search_action} can be considered as running OGD with unbiased estimates of gradients on the sequence of loss functions $\tilde{\ell}_t(w)$. For an unbiased estimate of the gradient, $\frac{1}{\delta} e(w_t^T x_t +\delta e - y_t)^2 x_t$, its norm is bounded as $(C^2 + 1)\mathcal{X}/\delta$. Note that different from Alg.~\ref{alg:random_search}, here the maximum norm of the unbiased gradient \emph{is independent of feature dimension $d$}. Now we apply Lemma 3.1 from \cite{flaxman2005online} on $\tilde{\ell}_t$, to get:
\begin{align}
\label{eq:tilde_random_action}
    \mathbb{E}\left[ \sum_{t=0}^T \tilde{\ell}_t(w_t)\right] - \min_{w^*\in\mathbb{W}}\sum_{t=0}^T \tilde{\ell}_t(w^*) \leq \frac{\mathcal{W}(C^2 + 1)\mathcal{X}}{\delta}\sqrt{T}
\end{align}
Again we can bound the difference between $\tilde{\ell}_t(w)$ and $\ell_t(w)$ for any $w$ using the fact that $(\hat{y}_t - y_t)^2$ is Lipschitz continuous with respect to prediction $\hat{y}_t$ with Lipschitz constant $C$: 
\begin{align}
    |\tilde{\ell}_t(w) - &\ell_t(w)| \nonumber\\&= |\mathbb{E}_{\epsilon\sim [-1,1]} [(w^Tx_t + \delta\epsilon - y_t)^2 - (w^T x_t - y_t)^2]|  \nonumber\\
    &\leq \mathbb{E}_{\epsilon\sim [-1,-1]}[C\delta |\epsilon|] \leq C\delta
\end{align} Substitute the above inequality back to Eq.~\ref{eq:tilde_random_action}, rearrange terms:
\begin{align}
    \mathbb{E}\left[\sum_{t=0}^T \tilde{\ell}_t(w_t)\right] - \min_{w^*\in\mathbb{W}}\sum_{t=0}^T \tilde{\ell}_t(w^*) \leq &\frac{\mathcal{W}(C^2+1)\mathcal{X}}{\delta}\sqrt{T} \nonumber\\&+ 2C\delta T
\end{align}
Set $\delta = T^{-0.25}\sqrt{\frac{\mathcal{W}(C^2+1)\mathcal{X}}{2C}}$, we get:
\begin{align}
  \mathbb{E}\left[\sum_{t=0}^T \tilde{\ell}_t(w_t)\right] - \min_{w^*\in\mathbb{W}}\sum_{t=0}^T \tilde{\ell}_t(w^*) \leq \sqrt{\mathcal{W}(C^2+1)\mathcal{X}C}T^{3/4}.  
\end{align}
\end{proof}

In summary, we showed that in general online learning, the advantage of random exploration in action space (i.e., Alg.~\ref{alg:random_search_action} in the linear contextual bandit setting) compared to random exploration in parameter space (i.e., Alg.~\ref{alg:random_search} in pure bandit setting) comes from the fact that the regret bound of Alg.~\ref{alg:random_search_action} does not explicitly depend on the feature dimension $d$.\footnote{Note that if we consider multivariate regression problem, than the regret bound will explicitly depend on the dimension of the action space. }

\subsection{Analysis on RL}
\label{sec:analysis-rl}
We now consider the general RL setting. For simplicity, we just focus on finite horizon control problems with deterministic dynamics $x_{t+1} = f(x_t, a_t)$, fixed initial position $x_1$, one-step cost function $c(x,a)$ and horizon $H$. We assume $x\in \mathbb{X}\subseteq\mathbb{R}^d$ and $a\in \mathbb{A}\subseteq \mathbb{R}$, where $d$ is large. Namely we consider control problem with high feature dimension $d$ and low action dimension $(1D)$.  Given any sequence of actions $\mathbf{a} = {a_1, ..., a_H}$, the total cost is fully determined by the sequence of actions:
\begin{align}
    & J(\mathbf{a}) =  \sum_{t=1}^H c(x_t, a_t) \\
    & s.t., x_{t+1} = f(x_t, a_t),  t \in [1,H-1].
\end{align}
Given any $\mathbf{a}\in\mathbb{R}^H$, with the knowledge of dynamics $f$, one can easily compute the gradient of $J$ with respect $\mathbf{a}$, denoted as $\nabla_{\mathbf{a}}J(\mathbf{a})$. With $\nabla_{\mathbf{a}}J(\mathbf{a})$, one can perform gradient descent on the action sequence, which has been used in the trajectory optimization literature. 

In the model-free setting, we are no longer able to  exactly compute  $\nabla_{\mathbf{a}}J(\mathbf{a})$. However, we can again use \emph{exploration in the action space} to form an estimation of $\nabla_{\mathbf{a}}J(\mathbf{a})$. Given $\mathbf{a}$, we sample $u_H\sim \mathbb{S}_H$ (i.e., $u_H$ is uniformly sampled from a $H$-dim unit sphere),\footnote{Uniformly sampling from a $H-$dim unit sphere can be implemented by first sampling $u_t$ from normal distribution for each time step $t$, and then forming the perturbation vector as  $[u_1,...,u_H]^T/ \sqrt{\sum_t u_t^2}$. } with a small $\delta\in\mathbb{R}^+$, we can formulate an estimation of $\nabla_{\mathbf{a}}J(\mathbf{a})$ as:
\begin{align}
\label{eq:open_loop}
   \tilde{\nabla}_{\mathbf{a}}J(\mathbf{a})  = \frac{H}{\delta} J(\mathbf{a} + \delta u_H) u_H,
\end{align} where $\tilde{\nabla}_{\mathbf{a}}J(\mathbf{a}) \in\mathbb{R}^H$. Again, we can show that $\tilde{\nabla}_{\mathbf{a}}J(\mathbf{a})$ is an unbiased estimate of the gradient of a smoothed version of $J$: $\mathbb{E}_{v\sim \mathbb{B}_H} [J(\mathbf{a} + \delta v)$], where $\mathbb{B}_H$ is a $H$-dim unit ball.

Now let us take policy into consideration. We assume parameterized deterministic linear policy $w^T x$ that takes state $x$ as input and outputs an action deterministically. At any state $x$, the Jacobian of the policy with respect to parameter $w$ is simply $x$. A given $w$ fully determines the total cost:
\begin{align}
&J(w) = \sum_{t=1}^H c(x_t, a_t) \\
& s.t., a_t = w^T x_t, x_{t+1} = f(x_t, a_t), t \in [1,H-1]. 
\end{align} If we know the model and cost function, we can exactly compute the gradient of $J(w)$ with respect to $w$. In model-free setting, we have two ways to estimate the $\nabla_{w}J(w)$. 

The first approach uses random exploration in action space. 
Given $w$, we first execute $\pi(x) = w^T x$ on the real system to generate a trajectory $\tau = {x_1,a_1,...,x_H, a_H}$. Denote $\mathbf{a} = {a_1,...,a_H}$, we can compute $\tilde{\mathbf{a}} = \mathbf{a} + \delta u_H$, and execute $\tilde{\mathbf{a}}$ in a open-loop manner and receive $J(a+\delta u_H)$ at the end of the simulation, from which we compute $\tilde{\nabla}_{\mathbf{a}}J(\mathbf{a})$ as shown in  Eq.~\ref{eq:open_loop}. Now by chain rule and using the Jacobians of the policy, we can estimate the gradient $\nabla_{w}J(w)$ as:
\begin{align}
\label{eq:action_space_rl}
    \tilde{\nabla}_{w}J(w) = X \tilde{\nabla}_{\mathbf{a}}J(\mathbf{a}) = \frac{H J(\tilde{\mathbf{a}})}{\delta} X u_H,
\end{align} where $X\in \mathbb{R}^{d\times H}$ and the $i$'th column of $X$ is the state $x_i$ along the trajectory $\tau$, and $u_H\sim \mathbb{S}_H$.  

The second approach estimates $\nabla_{w}J(w)$ by \emph{exploration in parameter space}. Denote $\tilde{w} = w + \delta u_d$, where $u_d\sim \mathbb{S}_{d}$, we can estimate $\nabla_{a}J(w)$ as:
\begin{align}
\label{eq:parameter_space_rl}
\hat{\nabla}_w J(w) = \frac{d J(\tilde{w})}{\delta} u_d,
\end{align} where we can show that $\hat{\nabla}_{w}J(w)$ is an unbiased estimate of the gradient of a smoothed version of $J(w)$: $\mathbb{E}_{v\sim \mathbb{B}_d}[J(w + \delta v_d )]$.

Unlike our previous setting of linear regression, the optimization problem in RL is non-convex and it is difficult to analyze its convergence. However, the norm of the gradient estimator still plays an important role and the regret is dependent on the maximum value of the norm.
In both $\tilde{\nabla}_w J(w)$ from Eq.~\ref{eq:action_space_rl} and $\hat{\nabla}_w J(w)$ from Eq.~\ref{eq:parameter_space_rl}, the terms $J(\tilde{a})$ and $J(\tilde{w})$ are at the similar scale, $\|u_d\| = \|u_H\| = 1$ as they are sampled from the corresponding unit sphere, and $\delta$ is a small number close to zero. When comparing the norm of both estimators $\hat{\nabla}_w J(w)$ and $\tilde{\nabla}_w J(w)$, we can see that the terms that really matter are $d$, $H$, and $\|X\|_F$. Again we see that the norm of the estimator $\hat{\nabla}_w J(w)$ explicitly depends on $d$--the feature dimension, while the norm of the estimator $\tilde{\nabla}_w J(w)$ does not explicitly depend on $d$, but instead the horizon length $H$ and the norm of the state $\|x\|$ which is problem dependent.  

So under what situations we may encounter that the estimator $\hat{\nabla}_w J(w)$ has smaller norm than the estimator $\tilde{\nabla}_w J(w)$ (i.e., random exploration in parameter space is preferred)? We can see that for problems where $H > d$,  which in fact is the case in most of continuous control tasks in OpenAI Gym, we could expect that random exploration in parameter space performs better than random exploration in action space. Another possible situation is that the dynamics under the current policy $w$ (i.e., $x_{t+1} = f(x_t, w^Tx_t)$) is unstable: the norm of the state $x_t$ grows exponentially with respect to $t$ along the trajectory, which corresponds to the LQR example that \citet{mania2018simple} demonstrated in their work. Such exponential dependency on horizon $H$ could lead to an estimator $\tilde{\nabla}_w J(w)$ with extremely large norm.

\subsection{Implementation Details}
\label{sec:implementation-details}

\subsubsection{Tuning Hyperparameters for ARS}
We tune the hyperparameters for ARS \cite{mania2018simple} in both MNIST and linear regression experiments, by choosing a candidate set of values for each hyperparameter: stepsize, number of directions sampled, number of top directions chosen and the perturbation length along each direction. The candidate hyperparameter values are shown in Table \ref{tab:hyperparam}. 

\begin{table}[t]
    \centering
    \begin{tabular}{|c|c|}
    \hline
    Stepsize &  $0.001, 0.005, 0.01, 0.02, 0.03$\\
    \hline
    \# Directions &  $10, 50, 100, 200, 500$\\
    \hline
    \# Top Directions & $5, 10, 50, 100, 200$\\
    \hline
    Perturbation & $0.001, 0.005, 0.01, 0.02, 0.03$ \\
    \hline
    \end{tabular} 
    \caption{Candidate hyperparameters used for tuning in ARS \textbf{V2-t} experiments}
    \label{tab:hyperparam}
\end{table}
           
We use the hyperparameters shown in Table \ref{tab:chosen-hyperparams} chosen through this tuning for each of the experiments in this work. The hyperparameters are chosen by averaging the test squared loss across three random seeds (different from the 10 random seeds used in actual experiments) and chosing the setting that has the least mean test squared loss after 100000 samples.

\begin{table}[t]
    \centering
    \begin{tabular}{|c|c|c|c|c|}
    \hline
    Experiment & Stepsize & \# Dir. & \# Top Dir. & Perturbation\\
    \hline
    MNIST     &  0.02 & 50 & 20 & 0.03\\
    \hline
    LR $d=10$     & 0.03 & 10 & 10 & 0.03 \\
    \hline
    LR $d=100$ & 0.03 & 10 & 10 & 0.02 \\
    \hline
    LR $d=1000$ & 0.03 & 200 & 200 & 0.03 \\
    \hline
    \end{tabular}
    \caption{Hyperparameters chosen for ARS \textbf{V2-t} in each experiment. LR is short-hand for Linear Regression.}
    \label{tab:chosen-hyperparams}
\end{table}

\begin{table}[t]
    \centering
    \begin{tabular}{|c|c|c|}
    \hline
    Experiment     &  Learning Rate & Batch size\\
    \hline
    MNIST     &  0.001 & 512\\
    \hline
    LR $d=10$ & 0.08 & 512\\
    \hline
    LR $d=100$ & 0.03 & 512\\
    \hline
    LR $d=1000$ & 0.01 & 512\\
    \hline
    \end{tabular}
    \caption{Learning rate and batch size used for REINFORCE experiments. We use an ADAM \cite{kingma2014adam} optimizer for these experiments.}
    \label{tab:hyperparam-reinforce}
\end{table}

\begin{table}[t]
    \centering
    \begin{tabular}{|c|c|c|}
    \hline
    Experiment     &  Learning Rate & Batch size\\
    \hline
    LR $d=10$ & 2.0 & 512\\
    \hline
    LR $d=100$ & 2.0 & 512\\
    \hline
    \end{tabular}
    \caption{Learning rate and batch size used for Natural REINFORCE experiments. Note that we decay the learning rate after each batch by $\sqrt{T}$ where $T$ is the number of batches seen.}
    \label{tab:hyperparam-nreinforce}
\end{table}

\subsubsection{MNIST Experiments}
\label{sec:mnist-details}

The CNN architecture used is as shown in Figure \ref{fig:arch}\footnote{This figure is generated by adapting the code from \url{https://github.com/gwding/draw_convnet}}. The total number of parameters in this model is $d=21840$. For supervised learning, we use a cross-entropy loss on the softmax output with respect to the true label. To train this model, we use a batch size of 64 and a stochastic gradient descent (SGD) optimizer with learning rate of 0.01 and a momentum factor of 0.5. We evaluate the test accuracy of the model over all the $10000$ images in the MNIST test dataset. 

\begin{figure}[t]
    \centering
    \includegraphics[width=\linewidth]{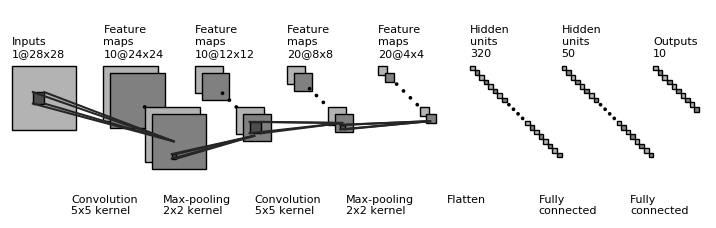}
    \caption{CNN architecture used for the MNIST experiments}
    \label{fig:arch}
\end{figure}

For REINFORCE, we use the same architecture as before. We train the model by sampling from the categorical distribution parameterized by the softmax output of the model and then computing a $\pm 1$ reward based on whether the model predicted the correct label. The loss function is the REINFORCE loss function given by,
\begin{equation}
    J(\theta) = \frac{1}{N} \sum_{i=1}^N r_i \log(\mathbb{P}(\hat y_i|x_i, \theta))
\end{equation}
where $\theta$ is the parameters of the model, $r_i$ is the reward obtained for example $i$, $\hat y_i$ is the predicted label for example $i$ and $x_i$ is the input feature vector for example $i$. The reward $r_i$ is given by $r_i = 2*\mathbb{I}[\hat y_i = y_i] - 1$, where $\mathbb{I}$ is the $0-1$ indicator function and $y_i$ is the true label for example $i$.

For ARS \textbf{V2-t}, we use the same architecture and reward function as before. The hyperparameters used are shown in Table \ref{tab:chosen-hyperparams} and we closely follow the algorithm outlined in \cite{mania2018simple}.

\subsubsection{Linear Regression Experiments}
\label{sec:linreg-details}

We generate training and test data for the linear regression experiments as follows: we sampled a random $d+1$ dimensional vector $w$ where $d$ is the input dimensionality. We also sampled a random $d \times d$ covariance matrix $C$. The training and test dataset consists of $d+1$ vectors $x$ whose first element is always $1$ (for the bias term) and the rest of the $d$ terms are sampled from a multivariate normal distribution with mean $\mathbf{0}$ and covariance matrix $C$. The target vectors $y$ are computed as $y = w^Tx + \epsilon$ where $\epsilon$ is sampled from a univariate normal distribution with mean $0$ and standard deviation $0.001$.

We implemented both SGD and Newton Descent on the mean squared loss, for the supervised learning experiments. For SGD, we used a learning rate of $0.1$ for $d=10, 100$ and a learning rate of $0.01$ for $d=1000$, and a batch size of 64. For Newton Descent, we also used a batch size of 64. To frame it as a one-step MDP, we define a reward function $r$ which is equal to the negative of mean squared loss. Both REINFORCE and ARS \textbf{V2-t} use this reward function. To compute the REINFORCE loss, we take the prediction of the model $\hat{w}^Tx$, add a mean $0$ standard deviation $\beta = 0.5$ Gaussian noise to it, and compute the reward (negative mean squared loss) for the noise added prediction. The REINFORCE loss function is then given by
\begin{equation}
    J(w) = \frac{1}{N} \sum_{i=1}^N r_i \frac{- (y_i - \hat{w}^Tx_i)^2}{2\beta^2}
\end{equation}
where $r_i = -(y_i - \hat y_i)^2$, $\hat y_i$ is the noise added prediction and $\hat{w}^Tx_i$ is the prediction by the model. We use an Adam optimizer with learning rate and batch size as shown in Table \ref{tab:hyperparam-reinforce}. For the natural REINFORCE experiments, we estimate the fisher information matrix and compute the descent direction by solving the linear system of equations $Fx = g$ where $F$ is the fisher information matrix and $g$ is the REINFORCE gradient. We use SGD with a $O(1/\sqrt{T})$ learning rate, where $T$ is the number of batches seen, and batch size as shown in Table \ref{tab:hyperparam-nreinforce}. 

For ARS \textbf{V2-t}, we closely follow the algorithm outlined in \cite{mania2018simple}.
\subsubsection{LQR Experiments}
\label{sec:LQR_details}
In the LQR experiments, we randomly generate linear dynamical systems $x_{t+1} = Ax_t + B a_t + \xi_t$, where $A\in\mathbb{R}^{100\times 100}$, $B\in\mathbb{R}^{100}$, $x\in\mathbb{R}^{100}$, $a\in\mathbb{R}^1$, and the noise $\xi_t \sim \mathcal{N}(0_{100}, cI_{100\times 100})$ with a small constant $c\in\mathbb{R}^+$.  We explicitly make sure that the maximum eigenvalue of $A$ is less than 1 to avoid instability. We fix a quadratic cost function $c(x,a) = x^T Q x + a R a$, where $Q = I_{100\times 100}$, and $R = 1e-3$. The stochastic linear policy is defined as a Gaussian distribution $\pi(\cdot|x;w,\sigma) = \mathcal{N}\left(w^T x, (\exp(\sigma))^2\right)$ where $w$ and $\sigma$ are parameters to learn. We initialize $w$ such that the resulting markov chain $x_{t+1} = (A+Bw^T)x_t$ is unstable ($(A+Bw^T)$ has eigenvalues larger than one). We simple set $\alpha = 0$, resulting the initial standard deviation to 1. 

\begin{table}[t]
    \centering
    \begin{tabular}{|c|c|}
    \hline
    Stepsize &  $0.005, 0.01, 0.02, 0.03$\\
    \hline
    \# Directions &  $20, 50, 100$\\
    \hline
    \# Top Directions & $10, 25, 50$\\
    \hline
    Perturbation & $0.01, 0.02, 0.03, 0.04$ \\
    \hline
    \end{tabular} 
    \caption{Candidate hyperparameters used for tuning in ARS \textbf{V1-t} in LQG at $H=50$ and $H= 100$}
    \label{tab:hyperparam_LQG}
\end{table}

We perform hyperparameter grid search for ARS V1-t at the specific horizon $H=50$  \footnote{We tuned hyperparameter for $H=100$ but found the same set of hyperparameters as the one for $H=50$.} using two random seeds (different from the test random seeds). The candidate parameters are listed in Table~\ref{tab:hyperparam_LQG} and then use the best set of hyperparameters for all $H$.


\end{document}